\documentclass[11pt]{article}

\usepackage{fullpage,times,url,natbib}

\usepackage{amsthm,amsfonts,amsmath,amssymb,epsfig,color,float,graphicx,verbatim}
\usepackage{algorithm,algorithmic}

\usepackage{hyperref}
\hypersetup{
	colorlinks   = true, 
	urlcolor     = blue, 
	linkcolor    = blue, 
	citecolor   = black 
}

\newtheorem{theorem}{Theorem}

\newtheorem{lemma}{Lemma}

\newtheorem{assumption}{Assumption}

\newcommand{\reals}{\mathbb{R}}
\newcommand{\E}{\mathbb{E}}

\newcommand{\bx}{\mathbf{x}}
\newcommand{\bw}{\mathbf{w}}

\newcommand{\bv}{\mathbf{v}}

\newcommand{\Ocal}{\mathcal{O}}

\newcommand{\Wcal}{\mathcal{W}}

\newcommand{\norm}[1]{\|#1\|}

\newcommand{\secref}[1]{Sec.~\ref{#1}}

\renewcommand{\eqref}[1]{Eq.~(\ref{#1})}
\newcommand{\lemref}[1]{Lemma~\ref{#1}}

\newcommand{\thmref}[1]{Thm.~\ref{#1}}

\title{Exponential Convergence Time of Gradient Descent\\ for One-Dimensional 
Deep Linear Neural Networks}
\date{}
\author{Ohad Shamir\\Weizmann Institute of Science}

\begin{document}

\maketitle

\begin{abstract}
	We study the dynamics of gradient descent on objective 
	functions of the form $f(\prod_{i=1}^{k} w_i)$ (with respect to scalar 
	parameters $w_1,\ldots,w_k$), which arise in the context of 
	training depth-$k$ linear neural networks. We prove that for standard 
	random initializations, and under mild assumptions on $f$, the number of 
	iterations required for convergence scales exponentially with the depth 
	$k$. We also show empirically that this phenomenon can occur in higher 
	dimensions, where each $w_i$ is a matrix. This highlights a potential 
	obstacle in understanding the convergence of gradient-based methods for 
	deep linear neural networks, where $k$ is large. 
\end{abstract}

\section{Introduction}

One of the biggest open problems in theoretical machine learning is to explain 
why deep artificial neural networks can be efficiently trained in practice, 
using 
simple gradient-based methods. Such training requires optimizing complex, 
highly non-convex objective functions, which seem intractable from a worst-case 
viewpoint. Over the past few years, much research has been devoted to this 
question, but it remains largely unanswered.

Trying to understand simpler versions of this question, significant attention 
has been devoted to \emph{linear} neural networks, which are predictors 
mathematically defined as $\bx\mapsto \prod_{i=1}^{k}W_i\bx$,
with $W_1,\ldots,W_k$ being a set of parameter matrices, and $k$ being the 
depth parameter
(e.g. 
\citet{saxe2013exact,kawaguchi2016deep,hardt2016identity,lu2017depth,bartlett2018gradient,
laurent2018deep}).
The optimization problem associated with training such networks can be 
formulated as
\begin{equation}\label{eq:linopt}
\min_{W_1,\ldots,W_k} F(W_1,\ldots,W_k) := 
f\left(\prod_{i=1}^{k}W_i\right)
\end{equation}
for some matrix-valued function $f$. Although much simpler than general 
feedforward neural networks (which involve additional non-linear functions), it 
is widely believed that \eqref{eq:linopt} captures important aspects of 
neural network optimization problems. Moreover, 
\eqref{eq:linopt} has a simple algebraic structure, which makes it more 
amenable 
to analysis. In particular, it is known that when $f$ is convex and 
differentiable, \eqref{eq:linopt} has no local minima except global ones (see 
\citet{laurent2018deep} and references therein). In other words, if an 
optimization algorithm converges to some local minimum, then it must converge 
to a global minimum.

Importantly, this no-local-minima result \emph{does not} imply that 
gradient-based methods indeed solve \eqref{eq:linopt} efficiently: Even when 
they converge to local minima (which is not always guaranteed, say in case the 
parameters diverge), the number of required iterations might be arbitrarily 
large. To study this question, \citet{bartlett2018gradient} recently considered 
the special case where 
$F(W_1,\ldots,W_k)~:=~\frac{1}{2}\|\prod_{i=1}^{k}W_i-Y\|_{F}^2$
(where $\norm{\cdot}_{F}$ is the Frobenius norm) for square matrices 
$W_1,\ldots,W_k,Y$, using gradient descent starting specifically from $W_i=I$ 
for all $i$. In this setting, the authors prove a polynomial-time convergence 
guarantee when $Y$ is positive semidefinite. On the other hand, when $Y$ is 
symmetric and with negative eigenvalues, it is shown that gradient descent with 
this initialization will never converge. Although these results provide 
important insights, they crucially assume
that each $W_i$ is initialized \emph{exactly} at the identity $I$. Since in 
practice parameters are initialized randomly, it is natural to ask whether 
such results hold with random initialization. Indeed, even though 
gradient descent might fail to converge with a specific initialization, it 
could be that even a tiny random perturbation is sufficient for polynomial-time 
convergence\footnote{For example, consider the 
objective $F(w_1,w_2)=(w_1 w_2+1)^2$ where $w_1,w_2\in \reals$. It is an easy 
exercise to show that gradient descent starting from any $w_1=w_2>0$ 
(and sufficiently small step sizes) will converge to the suboptimal saddle 
point $(0,0)$. 
On the other hand, polynomial-time convergence holds with random initialization 
(see \citet{du2018algorithmic}).}. More recently, \citet{arora2018convergence} 
considered gradient descent on a similar objective, and managed to prove strong 
polynomial-time 
convergence guarantees under certain assumptions about the initialization. 
However, as the authors discuss (in section 3.2.1), these assumptions are not 
generally satisfied for standard initialization approaches. In another recent 
related 
work, \citet{ji2018gradient} show that for certain classification problems on 
linearly separable data (corresponding to a suitable choice of $f$ in 
\eqref{eq:linopt}), gradient descent asymptotically 
converges to a globally optimal objective value. However, the result only 
applies to particular choices of $f$, and more importantly, is asymptotic and 
hence does not imply a finite-time convergence guarantee. Thus, analyzing the 
finite-time convergence of gradient descent on \eqref{eq:linopt}, with 
standard random initializations, remains a challenging open problem. 

In this paper, we consider a simpler special case of \eqref{eq:linopt}, where 
the matrices $W_1,\ldots,W_k$ are all scalars:
\begin{equation}\label{eq:linopt1}
\min_{\bw\in\reals^k} F(\bw) := f\left(\prod_{i=1}^{k} w_i\right)~.
\end{equation}
Our main and perhaps surprising result is that even in 
this relatively simple setting, gradient descent with 
random initialization can 
require $\exp(\Omega(k))$ iterations to converge. This holds under mild 
conditions on the function $f$, and with standard initializations (including 
Xavier initialization and any reasonable initialization close to 
$(1,\ldots,1)$). We complement this by showing that 
$\exp(\tilde{\Ocal}(k))\cdot \max\{1,\log(1/\epsilon)\}$ iterations are also 
sufficient for convergence to an $\epsilon$-optimal point. Moreover, in 
\secref{sec:experiments} we present experiments which 
strongly suggest that this phenomenon is not unique to one-dimensional 
networks, and at least in some cases, the same exponential dependence can also 
occur in  multi-dimensional networks (i.e., \eqref{eq:linopt} where each $W_i$ 
is a $d\times d$ matrix, $d>1$). The take-home message is that even if we focus 
on linear 
neural networks, natural objective functions without any spurious local minima, 
and random initializations, the associated optimization problems can sometime 
be intractable for gradient descent to solve, when the depth is large. 

Before continuing, we emphasize that our results do not imply that 
gradient-based methods cannot learn deep linear networks in general. What they 
do imply is that one would need to make additional assumptions or algorithmic 
modifications to circumvent these negative results: For example, explicitly 
using the fact that the matrix sizes are larger than $1$ -- something which is 
not clear how to do with current analyses -- or having a fine-grained 
dependency on the variance of the random initialization, as further discussed 
in \secref{sec:experiments}. Alternatively, our results might be circumvented 
using other gradient-based algorithms (for 
example, by adding random noise to the gradient updates or using adaptive 
step sizes), or other initialization strategies. However, that would not 
explain why plain gradient descent with standard random initializations is 
often practically effective on these problems. Overall, we believe our results 
point to a potential obstacle in understanding the convergence of 
gradient-based methods for linear networks: At the very least, one would have 
to rule out one-dimensional layers, or consider algorithms other than plain 
gradient descent with standard initializations, in order 
to establish polynomial-time convergence guarantees for deep linear networks. 

Finally, we note that our results provide a possibly interesting contrast to 
the 
recent work of \citet{arora2018optimization}, which suggests 
that increasing depth can sometimes accelerate the optimization process. Here 
we show that at least in some cases, the opposite occurs: Adding depth can 
quickly turn a trivial optimization problem into an intractable one for 
gradient descent.

\section{Preliminaries}

\textbf{Notation.} We use bold-faced letters to denote vectors. 
Given a vector $\bw$, $w_j$ refers to its $j$-th coordinate. 
$\norm{\cdot}$, 
$\norm{\cdot}_1$ and 
$\norm{\cdot}_{\infty}$ refer to the Euclidean norm, the $1$-norm and the 
infinity norm respectively. We let 
$\prod_{i=1}^{k}w_i$ and $\prod_i w_i$ be a 
shorthand for $w_1\cdot w_2\cdots w_k$. Also, we define a product 
over an empty set as being 
equal to $1$. Since our main focus is to 
study the dependence on the 
network 
depth $k$, we use the standard notation 
$\Ocal(\cdot),\Omega(\cdot),\Theta(\cdot)$ to hide constants independent of 
$k$, and $\tilde{\Ocal}(\cdot),\tilde{\Omega}(\cdot),\tilde{\Theta}(\cdot)$ to 
hide constants and factors logarithmic in $k$. 

\textbf{Gradient Descent.} We consider the standard gradient descent method for 
unconstrained optimization 
of functions $F$ in Euclidean space, which given an initialization point 
$\bw(1)$, performs repeated iterations of the form $\bw(t+1):=\bw(t)-\eta 
\nabla F(\bw(t))$ for $t=1,2,\ldots$ (where $\nabla F(\cdot)$ is the gradient, 
and $\eta>0$ is a step size parameter). For objectives as in 
\eqref{eq:linopt1}, 
we 
have $\frac{\partial}{\partial w_j} F(\bw)=
f'(\prod_i w_i)\prod_{j\neq i}w_i$, and gradient descent takes the form
\[
\forall j,~ w_j(t+1) = w_j(t)-\eta f'\left(\prod_i w_i(t)\right)\prod_{j\neq 
i}w_i(t)~.
\]

\textbf{Random Initialization.} One of the most common initialization methods
for neural networks is \emph{Xavier} initialization 
\citep{glorot2010understanding}, which in the setting of \eqref{eq:linopt} 
corresponds to choosing each entry of each $d\times d$ matrix $W_i$ 
independently from a zero-mean distribution with variance $1/d$ (usually 
uniform or Gaussian). This ensures 
that the variance of the network outputs (with respect 
to the initialization) is constant irrespective of the network size. Motivated 
by residual networks, 
\citet{hardt2016identity} and \citet{bartlett2018gradient} consider 
initializing each $W_i$ independently at $I$, possibly with some random 
perturbation. In this paper we denote such an initialization scheme as a 
\emph{near-identity} initialization. Since we focus here on the case $d=1$ as 
in \eqref{eq:linopt1}, 
Xavier initialization corresponds to choosing each $w_i$ independently from 
a zero-mean, unit-variance distribution, and 
near-identity initialization corresponds to choosing each $w_i$ close to $1$.

\section{Exponential Convergence Time for Gradient Descent}\label{sec:main}

For our negative results, we impose the following mild conditions on the 
function $f$ in \eqref{eq:linopt1}:
\begin{assumption}\label{assump:f}
		$f:\reals\rightarrow\reals$ is differentiable, Lipschitz continuous 
		and strictly monotonically 
		increasing on any interval $[-\frac{1}{2},z)$ where $z>0$. Moreover, 
		$\inf_{p\in [-\frac{1}{2},\infty)}f(p)-\inf_{p\in \reals}f(p) > 0$.
\end{assumption}
Here, we assume that $f$ is fixed, and our goal is to study the convergence 
time of gradient descent on \eqref{eq:linopt1} as a function of the depth $k$.
Some simple examples satisfying Assumption \ref{assump:f} in the context of 
machine learning include $f(x) = (x+1)^2$ and $f(x) = \log(1+\exp(x))$ (e.g.,
squared 
loss and logistic loss with respect to the input/output pair $(1,-1)$, 
respectively). We note that this non-symmetry with respect to positive/negative 
values is completely arbitrary, and one can prove similar results if their 
roles are reversed. 

\subsection{Xavier Initialization}

We begin with the case of Xavier initialization, where we initialize all 
coordinates of $\bw$ in \eqref{eq:linopt1} independently from a zero-mean, unit 
variance distribution. We will consider any distribution which satisfies the 
following:

\begin{assumption}\label{assump:initzero}
	$w_1(1),\ldots,w_k(1)$ are drawn i.i.d. from a zero-mean, unit variance 
	distribution such that 
	\begin{enumerate}
		\item $\Pr(w_1(1)\in [-a,a])\leq c_1 a$ for all $a\geq 0$
		\item $\E[|w_1(1)|]\leq 1-c_2$~
	\end{enumerate}
	where $c_1,c_2>0$ are absolute constants independent of $k$.
\end{assumption}

The first part of the assumption is satisfied for any 
distribution with bounded density. As to the second part, the 
following lemma shows that it is satisfied for uniform and Gaussian 
distributions (with an explicit $c_2$), and in fact for any non-trivial 
distribution (with a distribution-dependent $c_2$ -- see also footnote 
\ref{footnote:orthogonal}):
\begin{lemma}
	The following hold:
	\begin{itemize}
		\item If $w$ is drawn from a zero-mean, unit-variance Gaussian, then 
		$\E[|w|]< 0.8$~.
		\item If $w$ is drawn from a zero-mean, unit-variance uniform 
		distribution, then $\E[|w|]< 0.9$~.
		\item If $w$ is drawn from any zero-mean, unit variance distribution 
		other than uniform on $\{-1,+1\}$, then $\E[|w|]< 1$.
	\end{itemize}
\end{lemma}
\begin{proof}
	The first two parts follow from standard results on Gaussian and uniform 
	distributions. As to the third part, it is easy to see that the support of 
	any	distribution which satisfies the conditions cannot be a subset of 
	$\{-1,+1\}$, and therefore $w^2$ is not supported on a single value. By 
	Jensen's inequality and the fact that $\sqrt{\cdot}$ is a strictly concave 
	function, it follows that
	$\E[|w|]=\E[\sqrt{w^2}]<\sqrt{\E[w^2]} = 1$. 
\end{proof}

With such an initialization, we now show that gradient descent is 
overwhelmingly likely to take at least exponential time to converge:

\begin{theorem}\label{thm:zero}
	The following holds for some positive constants $c,c'$ independent of $k$: 
	Under 
	Assumptions 
	\ref{assump:f} and \ref{assump:initzero}, if gradient 
	descent is ran with any step size $\eta\leq \exp(ck)$, then with 
	probability at least 
	$1-\exp(-\Omega(k))$ over the initialization, the number of iterations 
	required to reach suboptimality less than $c'$ is at least 
	$\exp(\Omega(k))$.
\end{theorem}
In the above, $\Omega(\cdot)$ hides dependencies on the absolute constants in 
the theorem statement and the assumptions. The proof (as well as
all other major proofs in this paper) is presented in \secref{sec:proofs}. 

The intuition behind the theorem is quite simple: Under our assumptions, it is 
easy to show that the product of any $\Omega(k)$ coordinates from 
$w_1(1),\ldots,w_k(1)$ is overwhelmingly likely to be exponentially small in 
$k$. Since the derivative 
of our objective w.r.t. any $w_j$ has the form
$f'(\prod_i w_i)\prod_{i\neq j} w_i$,
it follows that the gradient is exponentially small in $k$. Moreover, we show 
that the gradient is exponentially small at any point within a bounded distance 
from the initialization (which is the main technical challenge of the proof, 
since the gradient is by no means Lipschitz). As a 
result, gradient descent will only make exponentially small steps. Assuming we 
start from a point bounded away from a global minimum, it follows that the 
number of required iterations must be exponentially large in $k$. 

We note that the observation that Xavier initialization leads to highly skewed 
values in deep enough networks is not new (see 
\citet{saxe2013exact,pennington2017resurrecting}), 
and has motivated alternative initializations such as orthogonal 
initialization\footnote{It is interesting to note that in our setting, 
orthogonal initialization amounts to choosing each $w_i$ in $\{-1,+1\}$, which 
can easily cause non-convergence, e.g. for $F(w_1,\ldots,w_k) = 
(\prod_i w_i-y)^2$ when $y\prod_i w_i(1)<0$ and small enough step 
sizes.\label{footnote:orthogonal}} Our 
contribution here is to rigorously analyze how this affects the optimization 
process for our setting.

\subsection{Near-Identity Initialization}

We now turn to consider initializations where each $w_i$ is initialized close 
to $1$. Here, it will be convenient to make deterministic rather than 
stochastic assumptions on the initialization point (which are satisfied with 
high probability for reasonable distributions):
\begin{assumption}\label{assump:initone}
	For some absolute constants $c_1,c_2,c_3>0$ independent of 
	$k$, gradient descent is initialized at a point $\bw(1)$ which satisfies 
	$\max_{j} |w_j(1)-1|~\leq~ k^{-c_1}$ and $c_2\leq\prod_i w_i(1)\leq c_3$.
\end{assumption}

To justify this assumption, note that if $w_1(1),\ldots, w_k(1)$ are chosen 
i.i.d. and not in the range of $1\pm k^{-c_1}$ for some $c_1>0$, then their 
product is likely to explode or vanish with $k$. 

\begin{theorem}\label{thm:identity}
	The following holds for some positive constants $c,c'$ independent of $k$:	
	Under 
	Assumptions \ref{assump:f} and \ref{assump:initone}, if gradient 
	descent is ran with any positive step size $\eta\leq c$, then the 
	number of 
	iterations required to reach suboptimality less than $c'$ is at least 
	$\exp(\Omega(k))$.
\end{theorem}
As before, $\Omega(\cdot)$ hides dependencies on the absolute constants in 
the theorem statement, as well as those in the assumptions.

The formal proof appears in \secref{sec:proofs}. To help explain its intuition, 
we provide in Figure 
\ref{fig:evolution} the actual evolution of $w_j(t)$ for a typical run of 
gradient descent, when $F(\bw)=F(w_1,\ldots,w_7) = 
\frac{1}{2}(\prod_{i=1}^{7} w_i+1)^2$ and we initialize all coordinates 
reasonably close to $1$. 
Recall that for any $w_j(t)$, the gradient descent updates take the form
\[
\forall j,~ w_j(t+1)~=~ w_j(t)-\eta\left(\prod_i w_i(t)+1\right)\prod_{i\neq 
j} w_i(t)~,
\]
where $\prod_i w_i(1)>0$. 
Thus, initially, all parameters $w_j(t)$ decrease with $t$, as to be expected. 
However, as their value fall to around or below $1$, their product decreases 
rapidly to 
$\exp(-\Omega(k))$. Since the gradient of each $w_j(t)$ scales as 
$\prod_{i\neq j} w_i(t)$, 
the magnitude of the gradients becomes very small, and the algorithm makes only 
slow progress.  Eventually, one of the parameters becomes negative, in which 
case all other parameters start increasing, and the algorithm 
converges. However, by a careful analysis, the length of the slow middle phase 
can be shown to be exponential in the depth / number of parameters $k$. 

\begin{figure}
\centering
\includegraphics[trim=2cm 0cm 0cm 0cm, clip=true, scale=1]{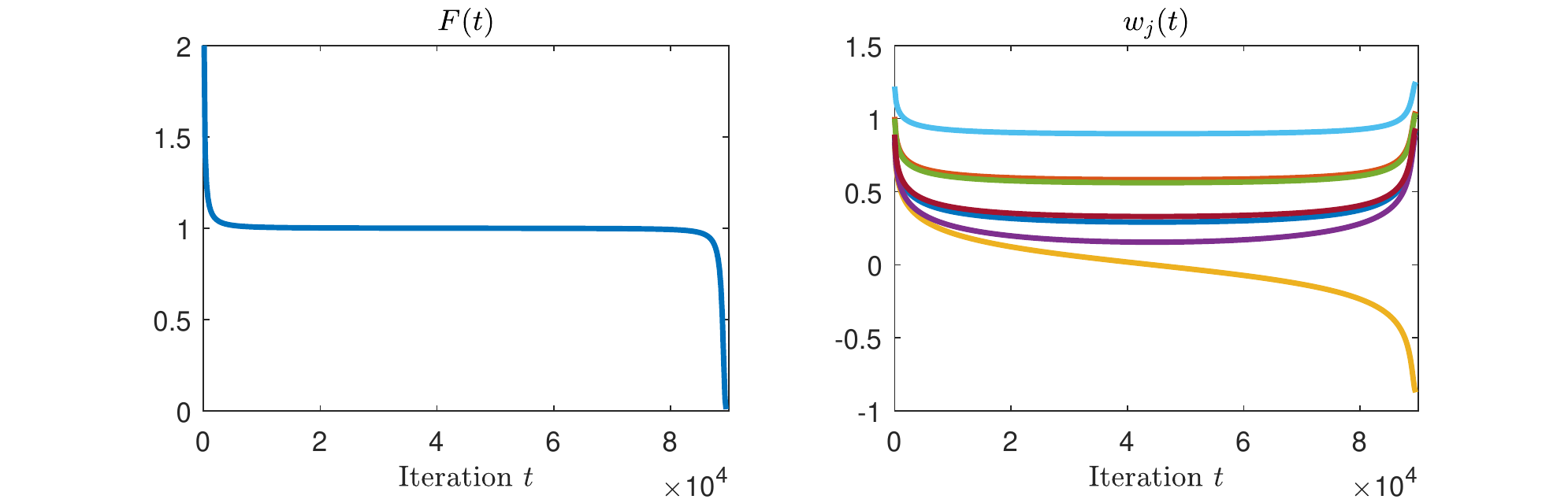}
\caption{The left figure plots $F(\bw(t))$ as a function of iteration $t$, for 
$F(\bw)=(\prod_{i=1}^{7} w_i+1)^2$. The right figure plots 
$w_1(t),w_2(t),\ldots,w_7(t)$ as a function of $t$. Best viewed in color.}
\label{fig:evolution}
\end{figure}

\subsection{A Positive Result}

Having established that the number of iterations is at least $\exp(\Omega(k))$, 
we now show that this is nearly tight. Specifically, we prove that gradient 
descent indeed converges in the settings studied so far, with a number of 
iterations scaling as $\exp(\tilde{\Ocal}(k))$ (this 
can be interpreted as a constant for any constant $k$). For simplicity, we 
prove this in the case where $f(\prod_i w_i) = \frac{1}{2}(\prod_i w_i-y)^2$, 
but the technique can be easily generalized to other convex $f$ under mild 
conditions. We note that the case of $y>0$ and each $w_i$ initialized to $1$ is 
covered by the results in \citet{bartlett2018gradient}. However, here we show a 
convergence result for other values of $y$, and even if $w_i$ are not all 
initialized at $1$. 

We will use the following assumptions on our objective and parameters:
\begin{assumption}\label{assump:pos}
The following hold for some absolute positive constants $c_1,c_2,c_3,c_4$
independent of $k$:
\begin{itemize}
\item $y=-c_1<0$
\item The initialization $w_1(1),\ldots,w_k(1)$ satisfies the following:
\begin{itemize}
	\item $|w_i(1)|\leq c_2$ and $\prod_i w_i(1)>y$
	\item $\min_{j\neq j'} \left||w_j(1)|-|w_{j'}(1)|\right|\geq k^{-c_4}$
	\item $\max_{j,j'} \left|\prod_{i\notin \{j,j'\}}w_i(1)\right|\leq c_4$
\end{itemize}
\end{itemize}
\end{assumption}

The assumptions $y<0$ and $\prod_i w_i(1)>y$ ensure that the objective 
satisfies the conditions of our negative results, for both 
Xavier and near-identity initializations (the other cases can be studied using 
similar techniques). 

\begin{theorem}\label{thm:posid}
	Consider the objective $F(\bw) = \frac{1}{2}\left(\prod_i w_i-y\right)^2$.
	Under Assumption \ref{assump:pos}, for any step size $\eta=k^{-c}$ for some 
	large enough constant $c>0$, and for any $\epsilon>0$, the number of 
	gradient descent iterations $t$ required for $F(\bw_t)\leq \epsilon$ is at 
	most $\exp\left(\tilde{\Ocal}(k)\right)\cdot\max\{1,\log(1/\epsilon)\}$. 
\end{theorem}

\section{Multi-Dimensional Networks}\label{sec:experiments}

So far, we showed that for one-dimensional linear neural networks, 
gradient descent can easily require exponentially many iterations (in the depth 
of the network) to converge. However, these results are specific to the case 
where the parameter matrix $W_i$ of each layer is one-dimensional, and do not 
necessarily extend to higher dimensions. A possibly interesting exception is 
when $F(W_1,\ldots,W_k)=\norm{\prod_i W_i - Y}_F^2$, and both $Y$ and the 
initialization $W_1(1),\ldots,W_k(1)$ are diagonal matrices. In that case, it 
is easy to show that the matrices produced by gradient descent remain diagonal, 
and the objective can be rewritten as a sum of independent one-dimensional 
problems for which our results would apply. However, this reasoning fails for 
non-diagonal initializations and target matrices $Y$. 

In this section, we study experimentally whether our theoretical results for 
one-dimensional networks might also extend to multi-dimensional ones. In 
particular, we consider the multi-dimensional generalization of the 
objective function studied earlier:
\[
F(W_1,\ldots,W_k) = \frac{1}{2}\left\|\prod_{i=1}^{k}W_i-Y\right\|_{F}^2,
\]
where $W_1,\ldots,W_k$ are $d\times d$ square matrices (for $d=25$), $Y=-I$ 
($I$ being 
the identity matrix), and $\|\cdot\|_{F}$ is the Frobenius norm. We 
ran gradient descent on this objective using three initialization strategies:
\begin{enumerate}
\item \emph{Xavier initialization}: Each entry of each matrix $W_i$ was 
initialized independently from a zero-mean Gaussian with variance
$\frac{1}{d}$.
\item \emph{Near-Identity initialization}: Each $W_i$ was initialized as $I+M$ 
where each entry of $M$ was sampled independently from a zero-mean Gaussian 
with variance $\frac{1}{dk}$. Up to numerical constants, this is the 
largest variance which ensures that 
$\E[(\prod_{i=1}^{k}W_i)(\prod_{i=1}^{k}W_i)^\top]$ 
remains bounded independent of $d,k$. To see this, note that had we used 
variance $\frac{c}{dk}$ for some constant $c$, then
$\E[W_iW_i^\top] = \left(1+\frac{c}{k}\right)I$ and thus 
$\E[(\prod_{i=1}^{k}W_i)(\prod_{i=1}^{k}W_i)^\top]
= \left(1+\frac{c}{k}\right)^k I\approx \exp(c)I$.
\item \emph{Near-Identity initialization with smaller variance}: Each $W_i$ was 
initialized as above, except that the variance of each entry in the matrix $M$ 
was $\frac{1}{(dk)^2}$. 
\end{enumerate}

\begin{figure}
\includegraphics[scale=0.7,trim=3.4cm 0cm 2cm 
0cm,clip=true]{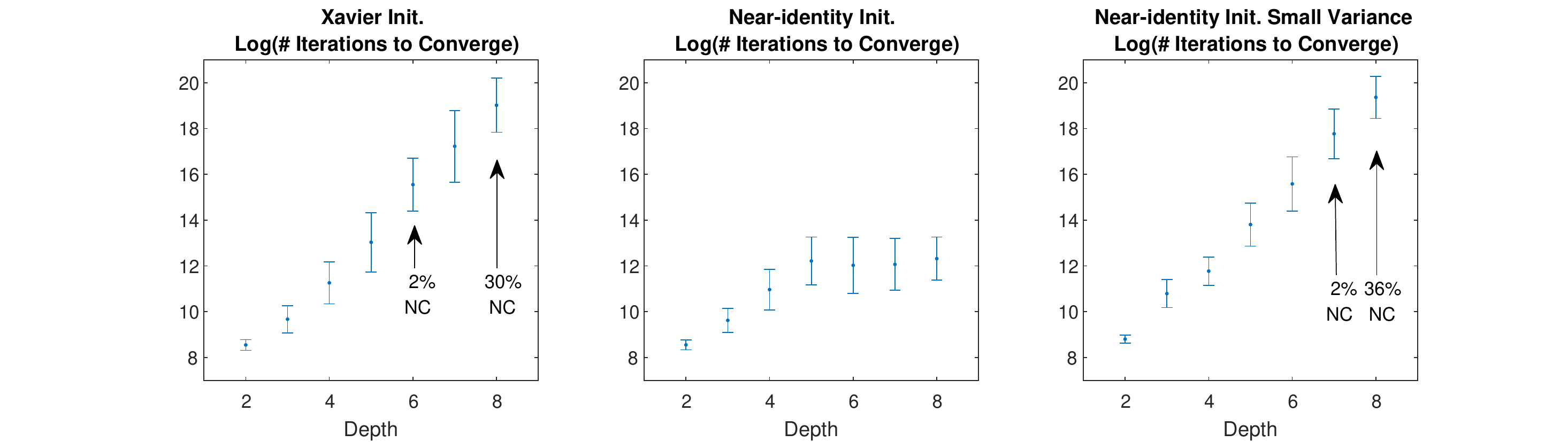}\label{fig:graphs}
\caption{Mean and Standard Deviation of the log number of iterations required 
for convergence, over $50$ trials, for each initialization strategy and depth 
parameter $k\in \{2,3,\ldots,8\}$. `NC' refers to the percentage of runs (for a 
given initialization and depth) which did not converge after $10^9$ iterations, 
if any. Note that when some trials did not converge, the bars actually 
under-estimate the mean convergence time if all trials were ran till 
convergence (since they only represent runs which took a log number of 
iterations less than $\log(10^9)=20.72...$). }
\end{figure}

For each random initialization strategy, and for depth parameter $k\in 
\{2,3,\ldots,8\}$ , we ran 50 trials of gradient 
descent, with a step size\footnote{Our results did not seem to change 
significantly by taking other bounded step sizes.} of $10^{-3}$, until either 
one 
of the following two 
stopping conditions occured:
\begin{itemize}
\item The objective value dropped 
below $0.1$ (or equivalently, $\norm{\prod_{i=1}^{k}W_i-Y}_F\leq 
\sqrt{1/5}$, a 
rather mild requirement).
\item The number of iterations exceeded $10^9$ iterations, in which case the 
algorithm was deemed to have failed to converge (note that from a practical 
viewpoint, one billion 
iterations is exceedingly large considering our problem size).
\end{itemize}
In Figure \ref{fig:graphs}, we plot the mean and standard deviation for the 
\emph{logarithm} of the number of iterations required to make the objective 
value less than $0.1$ (among the $50$ trials which converged). We also point 
out the percentage of trials which did not converge, if any. 

The figure strongly suggests that using both Xavier initialization and 
near-identity initialization with small variance, the required runtime scales 
exponentially with the depth (recall that the $y$-axis is in log scale). This 
indicates that the phenomenon of exponential scaling with depth is not just an 
artifact of one-dimensional networks, and can also occur in multi-dimensional 
networks, even with reasonable random initializations. On the flip side, when 
performing near-identity initialization with a large enough variance, we did 
not observe such an exponential scaling (as evidenced in the middle plot in the 
figure). Moreover, based on some additional experiments with other objective 
functions, it appears that although gradient descent can sometime require 
exponential time to converge, this phenomenon is not particularly common. A 
possible explanation to this is that in one dimension, $\prod_i w_i$ had to 
change sign, and hence pass through zero (see Figure \ref{fig:evolution}). This 
brought the iterates to a ``flat'' region with exponentially small gradients. 
In contrast, in multiple 
dimensions, to continuously change $\prod_i W_i$ from a matrix to some other 
matrix, it is always possible to go ``around'' any particular point. Our 
experiments suggest that gradient descent indeed avoids problematic flat 
regions in many cases, but not always. Overall, it seems quite possible that 
for multi-dimensional networks, the exponential runtime dependence on the 
depth can be avoided under reasonable assumptions -- however, some such 
assumptions would be necessary, and would need to exclude either objectives of 
the type we studied here, or some of the initializations. For example, such an 
analysis might need to explicitly separate between one-dimensional and 
multi-dimensional networks, or between near-identity initialization with 
variance $1/dk$ and with variance $1/(dk)^2$ (which are both polynomially large 
in $d,k$), and how to do so with existing analyses is currently unclear.

\section{Proofs}\label{sec:proofs}

\subsection{Proof of \thmref{thm:zero}}

The proof is based on the following two lemmas:

\begin{lemma}\label{lem:smallinit}
	Suppose $w_1,\ldots,w_k$ are drawn i.i.d. from a distribution such that 
	$\E[|w_1|]\leq a$ for some $a>0$. Then
	\[
	\Pr\left(\max_j \left|\prod_{i\neq j}w_i\right|\geq 
	k a^{(k-1)/2}\right)~\leq~
	a^{(k-1)/2}
	\]
\end{lemma}

\begin{proof}
	For any fixed $j$, by Markov's inequality and the i.i.d. assumption,
	\[
	\Pr\left(\left|\prod_{i\neq j}w_i\right|\geq ka^{(k-1)/2}\right)
	~\leq~
	\frac{\E\left[|\prod_{i\neq j}w_i|\right]}{ka^{(k-1)/2}}
	~=~ \frac{\left(\E[|w_1|]\right)^{k-1}}{ka^{(k-1)/2}}
	~\leq~ \frac{a^{k-1}}{ka^{(k-1)/2}} ~=~ \frac{1}{k}a^{(k-1)/2}~.
	\]
	Taking a union bound over all $j=1,2,\ldots,k$, the result follows.
\end{proof}

\begin{lemma}\label{lem:flatball}
	Let $\alpha,\beta,\delta>0$ be fixed. Let $\bw\in \reals^k$ such that 
	$\max_j \left|\prod_{i\neq j} 
	w_i\right|\leq 
	\alpha$ and $\min_i |w_i|\geq \delta$. Then for any $\bv$ such that 
	$\norm{\bv-\bw}\leq 
	\frac{\delta}{\sqrt{k-1}} 
	\log(\beta/\alpha)$, it holds that $|\prod_i v_i|\leq 
	\beta\norm{\bv}_{\infty}$ as well as
	$\norm{\nabla F(\bv)}~\leq~ 
	\sup_{p:|p|\leq \beta\norm{\bv}_{\infty}}|f'(p)|\cdot \sqrt{k}\beta$.
\end{lemma}
\begin{proof}
	We claim that it is enough to prove the following: 
	\begin{align}
	\forall \bw,\bv\in\reals^k~~&\text{s.t.}~~\max_j 
	\left|\prod_{i\neq j} 
	w_i\right|\leq 
	\alpha~~,~~\min_i |w_i|\geq \delta~~,~~
	\max_j\left|\prod_{i\neq j}v_i\right|> \beta\notag\\
	&\text{it holds that} ~~
	\norm{\bv-\bw}~>~ 
	\frac{\delta}{\sqrt{k-1}}\log(\beta/\alpha)~.\label{eq:toshow0}
	\end{align}
	Indeed, this would imply that for any $\bw$ satisfying the conditions 
	above, and 
	any $\bv$ 
	s.t. $\norm{\bv-\bw}\leq \frac{\delta}{\sqrt{k-1}}\log(\beta/\alpha)$, 
	we must 
	have 
	$\max_j \left|\prod_{i\neq j}v_i\right|\leq\beta$, 
	and therefore $|\prod_i v_i|\leq \beta\norm{\bv}_{\infty}$, as well as  
	$\norm{\nabla 
		F(\bv)}=\sup_{p:|p|\leq 
		\beta\norm{\bv}_{\infty}}|f'(p)|\cdot\norm{(\prod_{i\neq 
		1}v_i,\ldots,\prod_{i\neq k}v_i)}\leq 
		\sup_{p:|p|\leq \beta\norm{\bv}_{\infty}}|f'(p)|\sqrt{k}\beta$ 
		by definition of $F$, as 
		required. 
	
	To prove \eqref{eq:toshow0}, we first state and prove the following 
	auxiliary result:
	\begin{align}
	\forall \bw,\bv\in\reals^{k-1}~~&\text{s.t.}~~
	\forall i~v_i\geq w_i\geq 0~~,~~\prod_{i} 
	w_i\leq 
	\alpha~~,~~\min_i w_i\geq \delta~~,~~
	\prod_{i}v_i> \beta\notag\\
	&\text{it holds that} ~~
	\norm{\bv-\bw}~>~ 
	\frac{\delta}{\sqrt{k-1}}\log(\beta/\alpha)~.\label{eq:toshow}
	\end{align}
	This statement holds by the following calculation:
	\begin{align*}
	\norm{\bv-\bw}~&\geq~ \frac{1}{\sqrt{k-1}} 
	\norm{\bv-\bw}_1~=~\frac{1}{\sqrt{k-1}}\cdot\sum_i 
	(v_i-w_i)\\
	&\stackrel{(*)}{\geq}~ \frac{1}{\sqrt{k-1}}\sum_i 
	w_i\left(\log(v_i)-\log(w_i)\right)~\geq~ 
	\frac{\delta}{\sqrt{k-1}}\sum_i\left(\log(v_i)-\log(w_i)\right)\\
	&=~
	\frac{\delta}{\sqrt{k-1}}\log\left(\frac{\prod_i v_i}{\prod_i 
	w_i}\right)~>~ 
	\frac{\delta}{\sqrt{k-1}}\log\left(\frac{\beta}{\alpha}\right)~,
	\end{align*}
	where $(*)$ is due to the fact that $\log(\cdot)$ is $1/z$-Lipschitz in 
	$[z,\infty)$, and the assumption that $v_i\geq w_i\geq 0$. 
	
	It remains to explain how \eqref{eq:toshow} implies \eqref{eq:toshow0}. 
	Indeed, let $\bw,\bv$ be any two vectors in $\reals^k$ which satisfy the 
	conditions of \eqref{eq:toshow0}. Now, suppose we transform them into 
	vectors $\bw',\bv'\in \reals^{k-1}$ by the following procedure:
	\begin{itemize}
		\item Change the sign of every $w_i$ and $v_i$ to be positive
		\item For any $i$ such that $v_i<w_i$, change $v_i$ to equal $w_i$.
		\item Drop a coordinate $j$ which maximizes $|\prod_{i\neq j} v_i|$.
	\end{itemize}
	It is easy to verify that the resulting vectors $\bw',\bv'$ satisfy the 
	conditions 
	of \eqref{eq:toshow}, and $\norm{\bv'-\bw'}\leq \norm{\bv-\bw}$. Therefore, 
	by \eqref{eq:toshow}, $\norm{\bv-\bw}\geq \norm{\bv'-\bw'}\geq 
	\frac{\delta}{\sqrt{k-1}}\log(\beta/\alpha)$ as required.
\end{proof}

With these two lemmas in hand, we turn to prove the theorem. By 
\lemref{lem:smallinit} and Assumption \ref{assump:initzero}, we have
\[
\Pr\left(\max_j \left|\prod_{i\neq j} w_i(1)\right|\geq 
\exp(-2C k)\right)~\leq~ 
\exp(-C' k)~.
\]
for some fixed constants $C,C'>0$ and any large enough $k$. 
Moreover, again by Assumption \ref{assump:initzero}, it holds for any $i$ that 
$\Pr(|w_i(1)|\leq \exp(-C k))\leq \Ocal(\exp(-C k))$, so by a 
union bound,
\[
\Pr(\min_i |w_i| < \exp(-C k))~\leq~ \Ocal(k\exp(-C k)).
\]
Finally, by Assumption \ref{assump:initzero}, Markov's inequality and a union 
bound,
\[
\Pr(\norm{\bw(1)}_{\infty}\geq \exp(C k))~\leq k\exp(-C k)
\]
Combining the last three displayed equations with a union bound, and 
applying \lemref{lem:flatball} (with $\alpha=\exp(- 2C k))$, 
$\beta=2\alpha$, and $\delta=\exp(-C k)$), we get the following: With 
probability at least $1-\exp(-C' k)-\Ocal(k\exp(-C 
k))-k\exp(-C k)=1-\exp(-\Omega(k))$ over the choice of $\bw(1)$, 
\begin{itemize}
\item $\norm{\bw(1)}_{\infty}\leq \exp(C k)$.
\item For any $\bv$ at a distance at most $\exp(-C 
k)\frac{\log(2)}{\sqrt{k-1}}$ from $\bw(1)$, we have
\[
\norm{\bv}_{\infty}~\leq~ 
\norm{\bw(1)}_{\infty}+\exp(-C 
k)\frac{\log(2)}{\sqrt{k-1}}~\leq~ \Ocal(\exp(C k))~,
\]
\[
\left|\prod_i v_i\right|~\leq~ \beta\norm{\bv}_{\infty} ~=~ 
2\exp(-2C k)\cdot \Ocal\left(\exp(C k)\right)~=~
\Ocal\left(\exp\left(-C k\right)\right)
\]
and
\begin{align*}
\norm{\nabla F(\bv)}~&\leq~ 
\sup_{p:|p|\leq \beta\norm{\bv}_{\infty}}|f'(p)|\cdot\sqrt{k}\beta~\leq~
\sup_{p:|p|\leq \Ocal\left(\exp\left(-C k\right)\right)}|f'(p)|\cdot2\sqrt{k}
\exp(-2C k)\\
&=~
\Ocal\left(\sqrt{k}\exp(-2C k)\right)~.
\end{align*}
\end{itemize}
This has two implications:
\begin{enumerate}
	\item Since the gradient descent updates are of the form 
	$\bw(t+1)=\bw(t)-\eta \nabla F(\bw(t))$, and we can assume $\eta\leq 
	\exp(C k/2)$ by the theorem's conditions, the number of iterations 
	required to get to a distance 
larger than $\exp(-C k)\frac{\log(2)}{\sqrt{k-1}}$ from $\bw(1)$ is at 
least
\[
\frac{\exp(-C k)\frac{\log(2)}{\sqrt{k-1}}}{\exp(C k/2)\cdot 
\Ocal(\sqrt{k}\exp(-2C k))}
~=~ \Omega\left(\frac{\exp(C k/2)}{k}\right)~,
\]
which is at least $\exp(\Omega(k))$ iterations.
	\item As long as we are at a distance smaller than the above, 
	$
	\left|\prod_i 
	v_i\right|~\leq~ \Ocal(\exp(-Ck))\leq \exp(-\Omega(k))
	$. In particular, $\prod_i v_i\geq 
	-1/2$ for large enough $k$, so by Assumption 
	\ref{assump:f} and definition of $F$, we have that 
	$F(\bv)-\inf_{\bv}F(\bv)$ is lower bounded by a constant 
	independent of $k$.
\end{enumerate}
Overall, we get that with probability at least $1-\exp(-\Omega(k))$, we 
initialize at some region in which all points are at least $\Omega(1)$ 
suboptimal, and at least $\exp(\Omega(k))$ iterations are required to escape 
it. This immediately implies our theorem. 

\subsection{Proof of \thmref{thm:identity}}

We begin with the following auxiliary lemma, and then turn to analyze the 
dynamics of gradient descent in our setting.

\begin{lemma}\label{lem:gm}
	For any positive scalars $\alpha,w_1,\ldots,w_k$ such that $\min_i w_i > 
	\alpha$, 
	\[
	\prod_{i}(w_i-\alpha) ~\leq~ 
	\left(\left(\prod_i w_i\right)^{1/k}-\alpha\right)^k~.
	\]
\end{lemma}
\begin{proof}
	Taking the $k$-th root and switching sides, the inequality in the lemma is 
	equivalent to proving
	\[
	\left(\prod_i (w_i-\alpha)\right)^{1/k}+\alpha~\leq~ 
	\left(\prod_{i}w_i\right)^{1/k}.
	\]
	Letting $a_i = w_i-\alpha$, and $b_i=\alpha$ for all $i$, the above is 
	equivalent to proving that
	\[
	\left(\prod_i a_i\right)^{1/k}+\left(\prod_i b_i\right)^{1/k}~\leq~ 
	\left(\prod_i (a_i+b_i)\right)^{1/k},
	\]
	namely that the sum of the geometric means of two positive sequences 
	$(a_i)$ and $(b_i)$ is at most the geometric mean of their sum $(a_i+b_i)$. 
	This follows from the superadditivity of the geometric mean (see 
	\citet[Exercise 2.11]{steele2004cauchy})
\end{proof}

\begin{lemma}\label{lem:gap}
	If $\min_i w_i(t)\geq C$ and $\prod_i w_i(t)\leq C'$ for some positive 
	constants $C,C'$, 
	then for any $j,j'$, 
	\[
	|w_j(t+1)^2-w_{j'}(t+1)^2|~\leq~ |w_j(t)^2-w_{j'}(t)^2|+C''\eta^2 
\left(\prod_i w_i(t)\right)^2~,
	\]
	where $C''$ is some constant dependent only on $C,C'$ and the function $f$.
\end{lemma}
\begin{proof}
	By definition,
	\begin{align*}
	w_j(t+1)^2&-w_{j'}(t+1)^2\notag\\
	&=~\left(w_j(t)-\eta f'\left(\prod_i 
	w_i(t)\right)\prod_{i\neq 
	j}w_i(t)\right)^2 - \left(w_{j'}(t)-\eta f'\left(\prod_i 
	w_i(t)\right)\prod_{i\neq 
	j'}w_i(t)\right)^2 \notag\\
	&=~ w_j(t)^2-w_{j'}(t)^2+\eta^2f'\left(\prod_i w_i(t)\right)^2\left(
	\left(\prod_{i\neq j}w_i(t)\right)^2-\left(\prod_{i\neq 
	j'}w_i(t)\right)^2\right)\\
	&=~ w_j(t)^2-w_{j'}(t)^2+\eta^2\left(\prod_i w_i(t)\right)^2\cdot 
	f'\left(\prod_i 
	w_i(t)\right)^2\left(\frac{1}{w_j(t)^2}-\frac{1}{w_{j'}(t)^2}\right)~.
	\end{align*}
	By assumption, $0\leq \prod_i w_i(t)\leq C'$ and $\max_j 
	\frac{1}{w_j(t)^2}\leq 
	\frac{1}{C^2}$. Therefore, by our assumptions on $f$, the displayed 
	equation above implies that
	\[
	|w_j(t+1)^2-w_{j'}(t+1)^2|~\leq~ |w_j(t)^2-w_{j'}(t)^2|+C''\eta^2 
	\left(\prod_i w_i(t)\right)^2
	\]
	for some constant $C''>0$ dependent on $C,C'$ and $f$ as required.
\end{proof}

\begin{lemma}\label{lem:decay}
	Suppose that at some iteration $t$, for some constant $C$ independent of 
	$k$, it holds that $\max_i w_i(t)\leq C$ and $\prod_i w_i(t)\leq \beta$ for 
	some $\beta\in (0,C)$. Then 
	after at most 
	$
	\tau\leq 1+\Ocal(1)\cdot \frac{\beta^{1/k-1}}{\eta k}
	$
	iterations, if $\min_{j} w_j(r)\geq 1/2$ for all 
	$r=t,t+1,\ldots,t+\tau$, then 
	\begin{itemize}
		\item Each $w_i(r)$ as well as $\prod_i w_i(r)$ monotonically decrease 
in $r=t,t+1,\ldots,t+\tau$
		\item For all $r=t,t+1,\ldots,t+\tau-1$, $\max_j 
|w_j(r+1)-w_j(r)|\leq \Ocal(1)\cdot\eta\beta$
		\item $\prod_i w_i(t+\tau) \leq 
		\beta\cdot \exp(-1)$.
	\end{itemize}
	In the above, $\Ocal(1)$ hides constants dependent only on $C$ and the 
	function $f$.
\end{lemma}

\begin{proof}
	If $\prod_i w_i(t)\leq \beta\cdot\exp(-1)$, we can pick $\tau=0$, and the 
	lemma trivially holds. Otherwise, let $\tau$ be the smallest (positive) 
	index such that $\prod_i w_i(t)\leq \beta\cdot\exp(-1)$ (if no such index 
	exists, take $\tau=\infty$, although the arguments below imply that $\tau$ 
	must be finite). Since we assume $w_i(r)$ for all $i$ are 
	positive, and $f$ is monotonically increasing,
	\[
	w_j(r+1)~=~ w_j(r)-\eta f'\left(\prod_i w_i(r)\right)\prod_{i\neq 
		j}w_i(r)~\leq~ w_j(r),
	\]
	so $w_j(r)$ monotonically decreases in $r$. Moreover, these are all 
	positive numbers by assumption, so $\prod_i w_i(r)$ monotonically decreases 
	in $r$ as well. This shows the first part of the lemma.
	
	As to the second part, the displayed equation above, the fact that $w_j(r)$ 
	and $\prod_i w_i(r)$ decrease in $r$, and our assumptions on $f$ imply that 
	for any $r<t+\tau$,
	\begin{align*}
	w_j(r+1)~&=~ w_j(r)-\eta f'\left(\prod_i w_i(r)\right)\prod_{i\neq 
	j}w_i(r)~=~w_j(r)-\frac{\eta}{w_j(r)} f'\left(\prod_i 
	w_i(r)\right)\prod_{i}w_i(r)\\
	&=~ w_{j}(r)-\Theta(1)\cdot \eta\beta~.
	\end{align*}
	where $\Theta(1)$ hides constants dependent only on $f$ and $C$. As to the 
	third part of the lemma, fix some $s<\tau$, and repeatedly apply the 
	displayed equation above for 
	$r=t,t+1,\ldots,t+s$, to get that
	that $w_j(t+s)\leq 
	w_j(t)-\Theta(1)\cdot\eta\beta s$ (which is still $\geq 1/2$ by the lemma 
	assumptions). In that case,
	\begin{align*}
	\prod_i w_i(t+s)~&\leq~ \prod_i 
	\left(w_i(t)-\Theta(1)\cdot\eta\beta s\right)
	~\stackrel{(*)}{\leq}\left(\beta^{1/k}-\Theta(1)\cdot\eta\beta s\right)^k
	~=~ \beta\left(1-\Theta(1)\cdot\eta\beta^{1-1/k}s\right)^k\\
	&\leq~\beta\exp\left(-\Theta(1)\cdot 
		\eta\beta^{1-1/k}s k\right)
	\end{align*}
	where $(*)$ follows from \lemref{lem:gm} and the fact that $\prod_i 
	w_i(t)\leq \beta$. The right hand side in turn is at 
	most $\beta\cdot \exp(-1)$ for any $s\geq C'\beta^{1/k-1}/\eta 
	k$ for some constant $C'$. In particular, if $\tau>1+C'\beta^{1/k-1}/\eta 
	k$, then by choosing $s$ 
	s.t. $\tau>s\geq C'\beta^{1/k-1}/\eta k$, we get that
	$\prod_i w_i(t+s) \leq \beta\cdot \exp(-1)$ even though $s<\tau$, which 
	contradicts the definition of $\tau$. Hence $\tau\leq 
	1+C'\beta^{1/k-1}/\eta 
	k$ as stated in the lemma.
\end{proof}

Combining \lemref{lem:gap} and \lemref{lem:decay}, we have the following:
\begin{lemma}\label{lem:closealways}
	For any constants $C>0$ and index $T$, if $\prod_i w_i(1)\leq C$ and 
	$w_i(t)\geq\frac{1}{2}$ for all $i=1,\ldots,k$ and $t=1,2,\ldots,T$, then 
	for all such $t$,
	\begin{itemize}
		\item Each $w_i(t)$ as well as $\prod_i w_i(t)$ 
		monotonically decrease in $t$.
		\item $\max_j |w_j(t+1)-w_j(t)|\leq \Ocal(1)\cdot\eta$
		\item $\max_{j,j'}|w_j(t)-w_{j'}(t)|~\leq~ 
		k^{-\Omega(1)}+\Ocal(1)\cdot \left(\eta^2+\frac{\eta}{k}\right)$~.
	\end{itemize}
	In the above, $\Ocal(\cdot)$ hides constants dependent only on $C$ and the 
	constants in Assumptions \ref{assump:f} and \ref{assump:initone}.
\end{lemma}
\begin{proof}
	The first two parts of the lemma follow from \lemref{lem:decay} and the 
	fact that by Assumption \ref{assump:initone}, $w_i(1)\leq 1+k^{-\Omega(1)} 
	\leq \Ocal(1)$. As to the last part, define 
	$t_0\leq t_1\leq \ldots\leq t_s$ (where $t_0=1$) as the first indices $\leq 
	T$ such that for all 
	$r=0,\ldots,s$, $\prod_i 
	w_i(t_r)\leq 
	(\prod_i w_i(1))\exp(-r)$ (where $s$ is taken to be as large as possible). 
	By 
	\lemref{lem:decay}, we have the following:
	\begin{itemize}
		\item For all $r=0,\ldots,s-1$, $|t_{r+1}-t_{r}|\leq 
		1+\Ocal(1)\cdot\frac{\exp(-r)^{1/k-1}}{\eta k}$~.
		\item $|T-t_{s}|\leq 1+\Ocal(1)\cdot\frac{\exp(-s)^{1/k-1}}{\eta k}$~.
		\item For all $r=0,\ldots,s-1$ and any $t_{r} \leq t \leq t_{r+1}$, 
		we have $\prod_i w_i(t)\leq \Ocal(1)\cdot\exp(-r)$.
	\end{itemize}
	Combining this with \lemref{lem:gap}, it follows that for any $j,j'$, and 
	any $r=0,\ldots,s-1$,
	\begin{align*}
	|w_j(t_{r+1})^2-w_{j'}(t_{r+1})^2|~&\leq~ |w_j(t_r)^2-w_{j'}(t_r)^2|
	+ \Ocal(1)\cdot\eta^2\exp(-2r)\cdot \left(1+\frac{\exp(-r)^{1/k-1}}{\eta 
	k}\right)\\
	&\leq~ |w_j(t_r)^2-w_{j'}(t_r)^2|+\Ocal(1)\cdot 
	\left(\eta^2\exp(-2r)+\frac{\eta 
	\exp(-r)}{k}\right)~,
	\end{align*}
	as well as
	\[
	|w_j(T)^2-w_{j'}(T)^2|~\leq~|w_j(t_s)^2-w_j(t_s)^2|+\Ocal(1)\cdot 
	\left(\eta^2\exp(-2s)+\frac{\eta 
		\exp(-s)}{k}\right)~.
	\]
	Repeatedly applying the last two displayed equations, and using Assumption 
	\ref{assump:initone}, we get that
	\begin{align*}
	|w_j(T)^2-w_{j'}(T)^2|~&\leq~ 
	|w_{j}(1)^2-w_{j'}(1)^2|+\Ocal(1)\cdot\left(\eta^2\sum_{r=0}^{s}\exp(-2r)+\frac{\eta}{k}\sum_{r=0}^{s}\exp(-r)\right)\\
	&\leq~ 
	k^{-\Omega(1)}+\Ocal(1)\cdot\left(\eta^2+\frac{\eta}{k}\right)~.
	\end{align*}
	Since $|w_j(T)^2-w_{j'}(T)^2|=|w_j(T)+w_{j'}(T)|\cdot 
	|w_j(T)-w_{j'}(T)|\geq |w_j(T)-w_{j'}(T)|$ (as we have $\min_i 
	w_i(T)\geq 1/2$ by assumption), we get that
	$|w_j(T)-w_{j'}(T)|\leq 
	k^{-\Omega(1)}+\Ocal(1)\cdot\left(\eta^2+\frac{\eta}{k}\right)$ as 
	required.
	\end{proof}
	
With \lemref{lem:closealways} in hand, we can now prove the theorem. Let 
$T$ be 
the largest index such that $\min_i w_i(t)\geq 1/2$ for all 
$t=1,2,\ldots,T$ (and $\infty$ if this holds for all $t$). It follows that 
$\prod_i 
w_i(t)\geq 0$, 
and therefore, by Assumption \ref{assump:f}, $F(\bw(t))-\inf_{\bw} F(\bw)$ is 
at least a constant independent of $k$ for all $t=1,2,\ldots,T$. Thus, to prove 
the theorem, it is 
enough to show that if $T<\infty$, then $T\geq \exp(\Omega(k))$.

By Assumption \ref{assump:initone} and \lemref{lem:closealways}, we have  
that $w_1(1)\geq 1-k^{-\Omega(1)}$, $|w_1(t+1)-w_1(t)|\leq \Ocal(1)\cdot\eta$, 
and 
$\max_j |w_j(t)-w_1(t)|\leq k^{-\Omega(1)}+\Ocal(1)\cdot 
\left(\eta^2+\frac{\eta}{k}\right)$. 
On the 
other hand, if 
$T<\infty$, then $\min_i w_i(T+1)<1/2$. Therefore, if $k$ is large enough and 
$\eta$ is small enough, there 
exists some iteration $t\leq T$ such that $w_j(t)\in [2/3,3/4]$ for all $j$. 
This means that $\prod_i w_i(t)\leq (3/4)^k = \exp(-\Omega(k))$. Thus, by 
\lemref{lem:decay} (with $\beta=\exp(-\Omega(k))$, from iteration $t$ till 
iteration $T$, each $w_j$ decreases 
by at most $\Ocal(1)\cdot \eta \beta\leq \exp(-\Omega(k))$ at each iteration. 
By assumption, at iteration 
$T+1$, there is some $w_j(T+1)< 1/2$, so we must have $T-t\geq 
(2/3-1/2)/\exp(-\Omega(k)) = 
\exp(\Omega(k))$ as required.

\subsection{Proof of \thmref{thm:posid}}

To prove the theorem, we first state and prove the following key lemma:
\begin{lemma}\label{lem:signswitch}
For any initialization $\bw(1)$ and any $(\sigma_1,\ldots,\sigma_k)\in 
\{-1,+1\}^k$, let $\bv(1),\bv(2),\ldots$ denote the iterates produced by 
gradient descent starting from $\bv(1):=(\sigma_1 w_1(1),\ldots,\sigma_k 
w_k(1))$, w.r.t. the function
\[
F_{\sigma}(\bv) := \frac{1}{2}\left(\prod_i v_i-\sigma y\right)^2~,
\]
where $\sigma:=\prod_i \sigma_i$. 
Then for any $t\geq 1$,
\[
\bv(t) = (\sigma_1 w_1(t),\ldots,\sigma_k w_k(t))~~~\text{and}~~~
F(\bw(t)) = F_{\sigma}(\bv(t))~.
\]
\end{lemma}
\begin{proof}
We prove the lemma by induction. The base case ($t=1$) is immediate from the 
definitions and the fact that
\[
F_{\sigma}(\bv(1)) = \frac{1}{2}\left(\prod_i \sigma_i w_i(1)-\sigma y\right)^2=
\frac{1}{2}\left(\sigma \prod_i w_i(1)-\sigma y\right)^2=
F(\bw(1))~.
\] 
Assuming that the induction hypothesis holds for $t$, and recalling that 
$\sigma=\prod_i \sigma_i$, we have for any $j\in 
\{1,\ldots,k\}$ that
\begin{align*}
v_j(t+1) &= v_j(t)-\left(\prod_i v_i(t)-\sigma y\right)\prod_{i\neq j}v_i(t) = 
\sigma_j w_j(t)-\sigma\left(\prod_i w_i(t)-y\right)\prod_{i\neq j}\sigma_i
w_i(t)\\
&= \sigma_j\left(w_j(t)-\left(\prod_i w_i(t)-y\right)\prod_{i\neq 
j}w_i(t)\right) = 
\sigma_j 
w_j(t+1)~.
\end{align*}
As a result,
\[
F_{\sigma}(\bv(t+1)) = \frac{1}{2}\left(\prod_i v_i(t+1)-\sigma y\right)^2 = 
\frac{1}{2}\left(\sigma \prod_i w_i(t+1)-\sigma y\right)^2 = F(\bw(t+1))~.
\]
This establishes the inductive step for $t+1$, hence proving the lemma. 
\end{proof}

The lemma implies that for 
studying the dynamics of gradient descent starting from any initial point 
$(w_1(1),\ldots,w_k(1))$, we can arbitrarily change the signs of its 
coordinates, as long as the sign of $y$ is changed accordingly. In particular, 
we will assume without loss of generality that all $w_1(1),\ldots,w_k(1)$ are 
positive (again, as long as the sign of $y$ is fixed accordingly). The proof 
then proceeds as follows:
\begin{itemize}
	\item The simplest case is when after the sign transformations, $y>0$. By 
	our assumptions, this implies that both $y$ and $\prod_i w_i(1)$ 
	switched from being negative (and satisfying $\prod_i w_i(1)>y$) to 
	positive, hence we now have $y>\prod_i 	w_i(1)>0$. In that case, 
	\lemref{lem:phase3} below implies that 	
	$\exp(\tilde{\Ocal}(k))\log(1/\epsilon)$ iterations suffice.
	\item The case $y<0$ (which by our assumptions, implies $y<0<\prod_i 
	w_i(1)$) is more involved: First, 
we show that after $t=\exp(\tilde{\Ocal}(k))$ iterations, one (and only one) of 
the coordinates of $\bw(t)$ becomes non-positive (\lemref{lem:phase1}). Then, 
we show that after at most one additional iteration, that non-positive 
coordinate 
becomes negative and bounded away from $0$ (\lemref{lem:phase2}), the other 
coordinates remaining strictly positive. By \lemref{lem:signswitch}, we can 
then argue that 
at that time point, the dynamics become the same as the scenario where $y>0$, 
and all coordinates of the iterate are strictly positive. Again applying 
\lemref{lem:phase3}, we get that $\exp(\tilde{\Ocal}(k))\log(1/\epsilon)$  
additional iterations suffice for convergence.
\end{itemize}

\subsubsection{The case $y> \prod_i w_i(1)>0$}

We will need the following auxiliary lemma:
\begin{lemma}\label{lem:logab}
For any $a>0$, $b\geq 0$, $\log(a+b)\leq \log(a)+\frac{b}{a}$.
\end{lemma}
\begin{proof}
Since $\log(1+z)\leq z$ for all $z\geq 0$,  we have
$\log(a+b)=\log(a(1+b/a))=\log(a)+\log(1+b/a)\leq \log(a)+b/a$. 
\end{proof}

\begin{lemma}\label{lem:phase3}
Fix some $\gamma\geq \delta>0$. Suppose that $y>0$, and gradient descent on $F$ 
is initialized at some $\bw(1)$ such that $\prod_i w_i(1)\in [0,y)$, 
$w_{j^*}(1)\geq \delta$ for some $j^*\in \arg\min_i w_i(1)$, and 
$w_j(1)\geq\gamma$ for all $j\neq j^*$. Assuming step size $\eta \leq 
\delta^2/2ky^2$, we have that $F(\bw(t))\leq \epsilon$ for any 
$t ~\geq~ 
\frac{\log(y^2/2\epsilon)}{k\delta^2 \gamma^{2(k-2)}\eta}$.
\end{lemma}
\begin{proof}
Let 
\[
\Wcal~:= \left\{\bw\in \reals^k~:~ \prod_i w_i(1)\in [0,y)~,~\min_i 
w_i(1)\geq \delta~,~\forall j\neq j^*~w_j(1)\geq\gamma\right\}
\]
denote the set of points in $\reals^k$ which satisfy the 
initialization conditions of the lemma. 

First, we show that if the step size $\eta$ is small enough, then gradient 
descent will remain in $\Wcal$ forever. For that, it is enough to show that  
for any $\bw\in\Wcal$, the update $\bw':=\bw-\eta\nabla F(\bw)$ produced by 
gradient descent is in $\Wcal$ as well. By definition of $\Wcal$, 
it is easily verified that $w'_i\geq 
w_i>0$ for all $i$, so the only non-trivial condition to verify is that 
$\prod_j 
w'_j< y$. To 
show this, we note that by \lemref{lem:logab},
\begin{align*}
\log\left(\prod_j w'_j\right) ~&=~ \sum_j \log(w'_j)~=~
\sum_j \log\left(w_j+\eta(y-\prod_i w_i)\prod_{i\neq j}w_i\right)\\
&\leq~
\sum_j \log(w_j)+\eta(y-\prod_i w_i)\sum_j\frac{\prod_{i\neq 
j}w_i}{w_j}\\
&=~
\log\left(\prod_j w_j\right)+\eta(y-\prod_i w_i)\left(\prod_i 
w_i\right)\sum_j\frac{1}{w_j^2}\\
&<~
\log\left(\prod_j w_j\right)+\eta (y-\prod_i w_i)y\sum_j \frac{1}{\delta^2}\\
&=~
\log\left(\prod_j w_j\right)+\eta \frac{yk}{\delta^2}\left(y-\prod_i 
w_i\right)~.
\end{align*}
Thus, to ensure that $\prod_j w'_j< y$ (or equivalently, $\log(\prod_j 
w'_j)< \log(y)$), it is enough to ensure that
\[
\log\left(\prod_j w_j\right)+\eta \frac{yk}{\delta^2}\left(y-\prod_i
w_i\right)~\leq~ \log(y)~.
\]
Rearranging the above, we require that
\[
\eta \frac{yk}{\delta^2}~\leq~ \frac{\log(y)-\log(\prod_j w_j)}{y-\prod_j 
w_j}~.
\]
By the mean value theorem and the fact that $\prod_j w_j< y$, the right hand 
side can be lower bounded by $\min_{z\in (0,y]} \log'(z) = 1/y$, so it is 
enough to require
\[
\eta \frac{yk}{\delta^2}~\leq~ \frac{1}{y}~~~\Rightarrow~~~
\eta~\leq~ \frac{\delta^2}{ky^2}~,
\]
which indeed holds by assumption.

Having established that gradient descent will remain in $\Wcal$ forever, we now 
establish that the objective $F$ has a $\frac{2ky^2}{\delta^2}$-Lipschitz 
gradient on $\Wcal$: Indeed, 
the Hessian of $F$ at any $\bw\in\Wcal$ can be easily verified to equal
\[
(\nabla^2 F(\bw))_{r,s} = \begin{cases} \left(\prod_i 
w_i-y\right)\frac{\prod_{i}w_i}{w_r w_s}+\frac{\prod_{i} w^2_i}{w_r w_s}& 
r\neq s\\
\frac{\prod_{i}w_i^2}{w_r^2}& r=s\end{cases}~.
\]
Since $\bw\in\Wcal$, it follows that  magnitude of each entry in the $k\times 
k$ Hessian is 
at most $y\cdot \frac{y}{\delta^2}+\frac{y^2}{\delta^2}=2y^2/\delta^2$, 
and therefore its spectral norm (which is at most the Frobenius norm) can be 
upper bounded by $2ky^2/\delta^2$. 

The final ingredient we need is that $F$ satisfies 
\[
\norm{\nabla F(\bw)}^2~\geq~ 2k\delta^{2}\gamma^{2(k-2)}F(\bw)
\] 
for any $\bw\in\Wcal$ (this type of inequality is known as the 
Polyak-\L{}ojasiewicz condition, which ensures linear convergence rates for 
gradient descent on possibly non-convex functions -- see 
\citet{polyak1963gradient,karimi2016linear}). This 
follows 
from $\norm{\nabla F(\bw)}^2$, by definition, being equal to 
\[
(\prod_i w_i-y)^2\sum_j\left(\prod_{i\neq j}w_i\right)^2 = 
2F(\bw)\sum_j\left(\prod_{i\neq j}w_i\right)^2 \geq 
2F(\bw)k\left(\delta\gamma^{k-2}\right)^2~.
\]

Collecting these ingredients, we can now perform a standard analysis using the 
Polyak-\L{}ojasiewicz condition: If we do a gradient step to get from 
$\bw\in\Wcal$ to 
$\bw'\in \Wcal$ (i.e. $\bw':=\bw-\eta\nabla F(\bw)$), and assuming $\eta\leq 
\delta^2/2ky^2$, then
\begin{align*}
F(\bw')~&\leq~ F(\bw)+\nabla 
F(\bw)^{\top}(\bw'-\bw)+\frac{ky^2}{\delta^2}\norm{\bw'-\bw}^2\\
&=~ F(\bw)-\eta\norm{\nabla F(\bw)}^2+\frac{ky^2}{\delta^2}\norm{\eta\cdot 
\nabla 
F(\bw)}^2\\
&=~ F(\bw)-\eta\left(1-\frac{ky^2}{\delta^2}\eta\right)\norm{\nabla 
F(\bw)}^2\\
&\leq~ F(\bw)-\eta\cdot \frac{1}{2}\cdot  
2k\delta^{2}\gamma^{2(k-2)}F(\bw)\\
&=~ \left(1-k\delta^{2}\gamma^{2(k-2)}\eta\right)F(\bw)~\leq~
\exp(-k\delta^{2}\gamma^{2(k-2)}\eta)F(\bw)~.
\end{align*}
Applying this inequality $t$ times, we 
get that
\[
F(\bw(t))~\leq~ \exp\left(-k\delta^2\gamma^{2(k-2)}\eta 
t\right)F(\bw(1))~\leq~\frac{y^2}{2} \exp\left(-k\delta^2\gamma^{2(k-2)}\eta 
t\right)~.
\]
Equating the bound above to the target accuracy $\epsilon$ and solving for $t$, 
the result follows.
\end{proof}

\subsubsection{The Case $y<0<\prod_i w_i(1)$}

We first state the following auxiliary lemma, which establishes that the gaps 
between coordinates are monotonically increasing under suitable assumptions.

\begin{lemma}\label{lem:gapincrease}
Fix some coordinate indices $j,j'$ and iteration $t$, and suppose that 
$w_j(t)\leq w_{j'}(t)$, $\min_i w_i(t)\geq 0$, and $y<0$. Then
$
w_{j'}(t)-w_j(t) ~\leq~ w_{j'}(t+1)-w_{j}(t+1)
$.
\end{lemma}
\begin{proof}
Dropping the $(t)$ index to simplify notation, we have by definition that 
$w_{j'}(t+1)-w_{j}(t+1)$ equals
\begin{align*}
	&\left(w_{j'}-\eta\left(\prod_i 
	w_i-y\right)\prod_{i\neq 
		j'}w_i\right)
	-\left(w_j-\eta\left(\prod_i w_i-y\right)\prod_{i\neq j}w_i\right)\\
	&= \left(w_{j'}-w_j\right)-\eta\left(\prod_i 
	w_i-y\right)\left(\prod_{i\neq j'}w_i-\prod_{i\neq j}w_i\right)\\
	&= \left(w_{j'}-w_j\right)-\eta\left(\prod_i 
	w_i-y\right)\left(w_{j}-w_{j'}\right)\prod_{i\notin\{j,j'\}}w_i
	~=~ \left(w_{j'}-w_j\right)\left(1+\eta\left(\prod_i 
	w_i-y\right)\prod_{i\notin\{j,j'\}}w_i\right)~.
\end{align*}
Since $y<0$ and $w_i\geq 0$ for all $i$, the above is at least $w_{j'}-w_j= 
w_{j'}(t)-w_j(t)$ as required. 
\end{proof}

\begin{lemma}\label{lem:phase1}
Suppose $y<0$, and that $\bw(1)$ has positive entries which satisfy the theorem 
assumptions. If $\eta\leq k^{-C}$ for some sufficiently large constant $C$, 
then the 
following hold for some 
iteration 
$t_0\leq 
\exp(\tilde{\Ocal}(k))/\eta$:
\begin{itemize}
\item There exists a unique $j^*=\arg\min_i w_i(t_0)$, and $-\Ocal(1)\cdot 
\eta\leq w_{j^*}(t_0)\leq 0$.
\item $\min_{j\neq j^*} w_j(t_0)\geq k^{-\Ocal(1)}$, $\max_{j\neq 
j^*}w_j(t_0)\leq \Ocal(1)$, and 
$\max_{j}\prod_{i\notin \{j,j^*\}}w_i(t_0)\leq \Ocal(1)$.
\end{itemize}
\end{lemma}
It is important to note that the constants hidden in the $\Ocal(\cdot)$ notation
do not depend on $\eta$ (although they may depend on $C$). 
\begin{proof}
By \lemref{lem:gapincrease} and the theorem assumptions, the following holds 
for all iterations 
$t=1,2,\ldots,T$ where $\min_{i,t< T} w_i(t)\geq 0$: There exists a unique
$j^*=\arg\min_i w_i(1)$, $w_{j^*}(t)$ remains the unique smallest value 
among $w_1(t),\ldots,w_k(t)$, and its distance from any other coordinate (which 
was initially $k^{-\Ocal(1)}$) is 
	monotonically increasing in $t$. In 
	particular, for any $t< T$, 
	$\min_{j\neq j^*}w_j(t)\geq k^{-\Ocal(1)}$. As 
	a result, recalling that $y<0$, we have for all $t< T$ that
	\begin{align*}
	w_{j^*}(t+1)~&=~ w_{j^*}(t)-\eta\left(\prod_i 
	w_i(t)-y\right)\prod_{i\neq 
	j^*}w_i(t)~\leq~ w_{j^*}(t)-\eta\left(\prod_i 
	w_i(t)-y\right)\left(k^{-\Ocal(1)}\right)^{k-1}\\
	&\leq~ w_{j^*}(t)+\eta y\exp(-\tilde{\Ocal}(k))~\leq~
	w_{j^*}(t)-\eta\cdot\exp(-\tilde{\Ocal}(k))~.
	\end{align*}
	We also assume that initially $w_{j^*}(1)\leq \Ocal(1)$. 
	Therefore, after at most $t_0=\exp(\tilde{\Ocal}(k))/\eta$ 
	iterations, we will have $w_{j^*}(t_0)\leq 0$ for the first time. 
	
	It remains to show that $w_{j^*}(t_0)\geq -\Ocal(1)\cdot\eta$, as well as 
	the second bullet in the lemma. To that end, we note 
	that up till iteration $t_0$, for any $j$, both $w_j(t)$ and $\prod_{i\neq 
	j} w_i(t)$ are monotonically decreasing in 
		$t$, and moreover, $t_0>1$ (since $w_{j^*}(t_0)\leq 0$ and we assume 
		$w_{j^*}(1)>0$). Thus, by Assumption \ref{assump:pos},
	\begin{align*}
		w_{j^*}(t_0) &= w_{j^*}(t_0-1)-\eta\left(\prod_i 
		w_i(t_0-1)-y\right)\prod_{i\neq 
		j^*}w_i(t_0-1)\\
		&> 
		0-\eta\left(\prod_i w_i(1)-y\right)\prod_{i\neq j^*}w_i(1)
		~\geq~ -\eta\cdot \Ocal(1)~.
	\end{align*}
	Using this inequality, we have for any $j\neq j^*$
	\[
	w_j(t_0)~=~ w_{j^*}(t_0)+(w_j(t_0)-w_{j^*}(t_0)) ~\geq~ 
	-\eta\cdot\Ocal(1)+k^{-\Ocal(1)}~,
	\]
	which is at least $k^{-\Ocal(1)}$ if $\eta \leq k^{-C}$ for some 
	sufficiently large constant $C$. Finally, since $w_j(t)$ for any $j\neq 
	j^*$ is positive and monotonically decreasing up to iteration $t_0$, we have
	$w_j(t_0)\leq w_j(1)\leq \Ocal(1)$ and $\prod_{i\notin 
	\{j,j^*\}}w_i(t_0)\leq\prod_{i\notin \{j,j^*\}}w_i(1)\leq 
	\Ocal(1)$ by Assumption \ref{assump:pos}.
\end{proof}

\begin{lemma}\label{lem:phase2}
Under the conditions of \lemref{lem:phase1},
\begin{itemize}
\item There exists a unique $j^*=\arg\min_i w_i(t_0+1)$, and 
$w_{j^*}(t_0+1)~\leq~ -\eta\cdot 
\exp(-\tilde{\Ocal}(k))$.
\item $\min_{j\neq j^*} w_j(t_0+1)\geq k^{-\Ocal(1)}$ and 
$\prod_{i\neq j^*}w_i(t_0+1)\leq \Ocal(1)$.
\end{itemize}
\end{lemma}
\begin{proof}
By \lemref{lem:phase1}, we have $\prod_i w_i(t_0)\leq 0$, as well as $\prod_i 
w_i(t_0)= w_{j^*}(t_0)\cdot w_j(t_0)\cdot \prod_{i\notin \{j,j^*\}}w_i(t_0)\geq 
-\Ocal(1)\cdot \eta$ (where $j$ is arbitrary). This
implies 
that for sufficiently small $\eta$, $\frac{y}{2}\leq \prod_i w_i(t_0)\leq 0$.  
By definition of the gradient descent update, it follows that 
$w_{j^*}(t_0+1)\leq 
w_{j^*}(t_0)$ and for all $j\neq j^*$, 
$w_j(t_0+1)\geq 
w_j(t_0)$, which implies that $j^*$ remains the unique coordinate with smallest 
value as we move from iteration $t_0$ to iteration $t_0+1$, as well as  
$\min_{j\neq j^*} w_j(t_0+1)\geq \min_{j\neq 
j^*}w_j(t_0)\geq 
k^{-\Ocal(1)}$. 

We now turn to prove $w_{j^*}(t_0+1)~\leq~ -\eta\cdot 
\exp(-\tilde{\Ocal}(k))$. Using the fact that $\frac{y}{2}\leq \prod_i 
w_i(t_0)\leq 0$ as noted earlier,
\[
w_{j^*}(t_0+1) = w_{j^*}(t_0)-\eta\left(\prod_i w_i(t_0)-y\right)\prod_{i\neq 
j^*}w_i(t_0)~\leq~ 0+\eta\cdot\frac{y}{2}\cdot\left(k^{-\Ocal(1)}\right)^{k-1}
~\leq~ -\eta\cdot \exp(-\tilde{\Ocal}(k))~.
\]

Finally, to prove $\prod_{i\neq j^*} w_i(t_0+1)\leq \Ocal(1)$, we have by 
definition that for any $j\neq j^*$,
\[
w_{j}(t_0+1) ~=~ w_{j}(t_0)-\eta\left(\prod_i 
w_i(t_0)-y\right)w_{j^*}(t_0)\cdot\prod_{i\notin \{j,j^*\}}w_i(t_0)~.
\]
Using the fact that
$\frac{y}{2}\leq \prod_i w_i(t_0)\leq 0$ as shown earlier, and noting that by  
\lemref{lem:phase1}, $|w_{j^*}(t_0)|\leq \Ocal(1)\cdot \eta$ and 
$\left|\prod_{i\notin \{j,j^*\}}w_i(t_0)\right|\leq \Ocal(1)$, it follows from 
the displayed equation above that
$w_{j}(t_0+1)\leq w_{j}(t_0)+\Ocal(1)\cdot \eta^2$. Therefore,
\[
\prod_{i\neq j^*} w_i(t_0+1)~\leq~ \prod_{i\neq 
j^*}\left(w_i(t_0)+\Ocal(1)\cdot \eta^2\right)~\leq~ \prod_{i\neq j^*} 
\left(w_i(t_0)\left(1+\frac{\Ocal(1)\eta^2}{w_i(t_0)}\right)\right)~.
\]
Since $\min_{j\neq j^*} w_j(t_0)\geq k^{-\Ocal(1)}$ (where the $\Ocal(1)$ does 
not depend on $\eta$), then by picking $\eta\leq 
k^{-C}$ for a sufficiently large $C$, the above is at most
$
\left(\prod_{i\neq 
j^*}w_i(t_0)\right)\left(1+\frac{\Ocal(1)}{k}\right)^{k-1}~\leq~ 
\left(w_j(t_0)\prod_{i\notin\{j,j^*\}} w_i(t_0)\right)\cdot \Ocal(1)~\leq~ 
\Ocal(1)$,
where we used \lemref{lem:phase1} and where $j$ is arbitrary.
\end{proof}

\subsubsection{Putting Everything Together}

As discussed at the beginning of the proof, we can assume w.l.o.g. that 
$w_1(1),\ldots,w_k(1)$ are all positive (and in fact, $\min_i w_i(1)\geq 
k^{-\Ocal(1)}$ by our assumptions), and only consider the cases $y>\prod_i 
w_i(1)>0$ and $y <0<\prod_i w_i(1)$. 

\begin{itemize}
\item If $y>\prod_i 
w_i(1)>0$, we can apply 
\lemref{lem:phase3} with $\gamma = \delta = k^{-\Ocal(1)}$ and any $\eta = 
k^{-c}$ for some large enough constant $c$, to get a convergence 
to an $\epsilon$-optimal solution in 
$\exp(\tilde{\Ocal}(k))\cdot\max\{1,\log\left(1/\epsilon\right)\}$
iterations.
\item If $y<0<\prod_i w_i(1)$, and assuming $\eta=k^{-c}$ for some large enough 
constant 
$c>0$, 
then \lemref{lem:phase1} and \lemref{lem:phase2} together tell
us that after at most $\exp(\tilde{\Ocal}(k))$ iterations, we get to an 
iteration $t=t_0+1$ where $w_{j^*}(t)\leq 
-\exp(-\tilde{\Ocal}(k))$ for some 
$j^*$, 
$w_{j}(t)\geq k^{-\Ocal(1)}$ for all $j\neq j^*$, and $0> 
\prod_i w_i(t)\geq 
-\Ocal(1)\cdot \eta\geq -\Ocal(1)\cdot k^{-\Omega(1)}>y$ for large enough $k$. 
Therefore, by 
\lemref{lem:signswitch}, the dynamics of gradient descent from this time point 
is identical to case where we switch the signs of $y$ and $w_{j^*}$, so that 
$y>0$, $w_{j^*}(t)\geq \exp(-\tilde{\Ocal}(k))$,  
$w_j(t)\geq k^{-\Ocal(1)}$ for all $j\neq j^*$, and $y>\prod_i w_i(t)>0 
$ for large enough $k$. Now applying 
\lemref{lem:phase3} with 
$\delta=\exp(-\tilde{\Ocal}(k))$, $\gamma=k^{-\Ocal(1)}$, and any step size 
$\eta=k^{-c}$ for some large enough $c$, we get that 
$\exp(\tilde{\Ocal}(k))\cdot\max\{1,\log(1/\epsilon)\}$ additional iterations 
suffice for 
convergence. Overall, $\exp(\tilde{\Ocal}(k))\cdot\max\{1,\log(1/\epsilon)\}$ 
iterations are 
sufficient.
\end{itemize}

\bibliographystyle{plainnat}
\bibliography{bib}

\end{document}